\documentclass[11pt]{article} 
\usepackage{url}

\usepackage{liyang}
\usepackage{framed}

\usepackage{amsthm,amsfonts,amsmath,amssymb,epsfig,color,float,graphicx,verbatim}
\usepackage{multirow}

\def\colorful{1}

\ifnum\colorful=1

\newcommand{\blue}[1]{{{\color{blue}#1}}}
\newcommand{\red}[1]{{\color{red} {#1}}}

\fi
\ifnum\colorful=0

\newcommand{\blue}[1]{{{#1}}}
\newcommand{\red}[1]{{{#1}}}

\fi

\usepackage{boxedminipage}

\newcommand{\ignore}[1]{}

\newcommand{\rnote}[1]{\footnote{{\bf \color{orange}Rocco:} {#1}}}

\newcommand{\flatten}[2]{({#1})^{{#2}}}

\newcommand{\norm}[1]{\left\|#1\right\|}

\title{
Near--Optimal Density Estimation in Near--Linear Time
Using Variable--Width Histograms}

\author{Siu-On Chan\\
Microsoft Research, New England\\
{\tt sochan@gmail.com}.\\
\and
Ilias Diakonikolas\thanks{Supported by EPSRC grant EP/L021749/1, and a Marie Curie Career Integration Grant.}\\
University of Edinburgh\\
{\tt ilias.d@ed.ac.uk}.\\
\and
Rocco A. Servedio\thanks{Supported by NSF grants CCF-0915929 and CCF-1115703.}\\
Columbia University\\
{\tt rocco@cs.columbia.edu}.\\
\and
Xiaorui Sun\thanks{Supported by NSF grant CCF-1149257.}\\
Columbia University \\
{\tt xiaoruisun@cs.columbia.edu}.
}

\begin{document}

\maketitle

\begin{abstract}
Let $p$ be an unknown and arbitrary probability distribution over $[0,1)$.  We consider the problem
of \emph{density estimation}, in which a learning algorithm is given i.i.d. draws from $p$ and must
(with high probability) output a hypothesis distribution that is close to $p$.  The main contribution of this paper is
a highly efficient density estimation algorithm for learning using a variable-width histogram, i.e., a
hypothesis distribution with a piecewise constant probability density function.

In more detail,  for any $k$ and $\eps$, we
give an algorithm that makes $\tilde{O}(k/\eps^2)$ draws from $p$, runs in $\tilde{O}(k/\eps^2)$ time, and outputs a
hypothesis distribution $h$ that is piecewise constant with $O(k \log^2(1/\eps))$ pieces.  With high probability the 
hypothesis $h$ satisfies $\dtv(p,h) \leq C \cdot \opt_k(p) + \eps$, where $\dtv$ denotes the total variation distance
(statistical distance), $C$ is a universal constant,
and $\opt_k(p)$ is the smallest total variation distance between $p$ and any $k$-piecewise constant distribution.
The sample size and running time of our algorithm are optimal up to logarithmic factors.
The ``approximation factor'' $C$  in our result is inherent in the problem, as we prove that
no algorithm with sample size bounded in terms of $k$ and $\eps$ can achieve $C<2$ regardless 
of what kind of hypothesis distribution it uses.
\end{abstract}

\section{Introduction}

Consider the following fundamental statistical task: {\em Given independent
draws from an unknown probability distribution, what is the minimum sample size needed to obtain an accurate estimate of the distribution?}
This is the question of {\em density estimation},  a classical problem in statistics with a rich history and an extensive literature
(see e.g.,~\cite{BBBB:72, DG85, Silverman:86,Scott:92,DL:01}).
While this broad question has mostly been studied from an  information--theoretic perspective, it is an inherently algorithmic question as well, 
since the ultimate goal is to describe and understand algorithms that are both computationally and information-theoretically efficient.
The need for computationally efficient learning algorithms\ignore{ has long been recognized by the statistics and machine learning communities, and} is only becoming more acute with the recent flood of data across the sciences;
the ``gold standard'' in this ``big data'' context is an algorithm with information-theoretically (near-) optimal sample size 
and  running time (near-) linear in its sample size.

\ignore{
An important class of ``big data'' problems can be modeled as a sample from a probability distribution over a very large (or even infinite) domain.}
In this paper we consider learning scenarios in which an algorithm is given an input data set which is a sample of i.i.d. draws from an unknown probability distribution.
It is natural to expect (and can be easily formalized) that, if the underlying distribution of the data is inherently ``complex'',  it may be hard to even approximately reconstruct the distribution.  
But what if the underlying distribution is ``simple'' or ``succinct'' -- can we then reconstruct the distribution to high accuracy in a 
computationally and sample-efficient way? In this paper we answer this
question in the affirmative for the problem of learning ``noisy'' {\em histograms}, arguably one of the most basic density estimation problems in the literature. 

To motivate our results, we begin by briefly recalling the role of histograms in density estimation.
Histograms constitute ``the oldest and most widely used method for density estimation''~\cite{Silverman:86},
first introduced by Karl Pearson in~\cite{Pearson}. Given a sample from a probability density function (pdf) $p$, the method partitions the domain into a number of 
intervals (bins) $B_1, \ldots, B_k$, and outputs the ``empirical'' pdf which is constant within each bin.  A
{\em $k$-histogram} is a piecewise constant distribution over bins $B_1, \ldots, B_k$, where the probability mass of each interval
$B_j$, $j \in [k]$, equals the fraction of observations in the interval. Thus, the goal of the ``histogram method'' 
is to approximate an unknown pdf $p$ by an appropriate $k$-histogram. 
It should be emphasized that the number $k$ of bins to be used and the
``width'' and location of each bin are unspecified; they are parameters of the estimation problem and are typically selected in an \emph{ad hoc} manner.
\ignore{
depending on the intricacies of the specific scenario.} 

We study the following distribution learning question:

\begin{quote}
{\em Suppose that there {\em exists} a $k$-histogram that provides an accurate approximation to the unknown target distribution.
Can we efficiently {\em find} such an approximation?}

\end{quote}

In this paper, we provide a fairly complete affirmative answer to this basic question.   Given a bound $k$ on the number of intervals, we 
give an algorithm that uses a near-optimal sample size, runs in {\em near-linear time} (in its sample size), and approximates the target distribution nearly as accurately as the best $k$-histogram.

To formally state our main result, we will need a few definitions.
We work in a standard model of learning an unknown probability distribution from samples, essentially that of~\cite{KMR+:94short}, 
which is a natural analogue of Valiant's well-known PAC model for learning Boolean functions~\cite{val84} 
to the unsupervised setting of learning an unknown probability distribution.\footnote{We remark that our model 
is essentially equivalent to the ``minimax rate of convergence under the $L_1$ distance'' in statistics~\cite{DL:01}, 
and our results carry over to this setting as well.}
A distribution learning problem is defined by a class ${\cal C}$ of distributions over a domain $\Omega$.
The algorithm has access to independent draws from an unknown pdf $p$, 
and its goal is to output a hypothesis distribution $h$ that is ``close'' 
to the target distribution $p$. We measure the closeness between distributions using the {\em statistical distance} or total variation distance.
In the ``noiseless'' setting, we are promised that $p \in {\cal C}$ and the goal is to construct a 
hypothesis $h$ such that (with high probability) the total variation distance
$ \dtv(h, p)$ between $h$ and $p$ is at most $\eps$, where $\eps>0$ is the accuracy parameter. 

The more challenging ``noisy'' or {\em agnostic}
model captures the situation of having arbitrary (or even adversarial) noise in the data. In this setting, 
we do not make any assumptions about the target density $p$ and the goal is to find a 
hypothesis $h$ that is almost as accurate as the ``best'' approximation of $p$ by any distribution in ${\cal C}$.
Formally, given sample access to a (potentially arbitrary) target distribution $p$ and $\eps>0$, the goal of an \emph{agnostic learning 
algorithm for ${\cal C}$} is to compute a hypothesis distribution $h$ such that $\dtv(h, p) \leq \alpha \cdot \opt_{{\cal C}}(p)+\eps$, 
where  $\opt_{{\cal C}}(p) := \inf_{q \in {\cal C}} \dtv(q, p)$ -- i.e.,  $\opt_{{\cal C}}(p)$ is the statistical distance between $p$ and the closest
distribution to it in ${\cal C}$ -- and $\alpha \ge 1$ is a constant (that may depend on the class ${\cal C}$)\ignore{ and the accuracy $\eps$)}.  
We will call such a learning algorithm an {\em $\alpha$-agnostic learning algorithm for ${\cal C}$};
when $\alpha>1$ we sometimes refer to this as a \emph{semi-agnostic learning algorithm}.

A distribution $f$ over a finite interval $I \subseteq \R$ is called {\em $k$-flat} if there exists a partition of $I$ into $k$ intervals
$I_1, \ldots, I_k$ such that the pdf $f$ is constant within each such interval. We henceforth (without loss of generality for densities with bounded support) restrict ourselves to the case
$I = [0,1)$.
Let ${\cal C}_k$ be the class of all \emph{$k$-flat} distributions over $[0,1)$.  
For a (potentially arbitrary) distribution $p$ over $[0,1)$ we will denote by $\opt_k(p): = \inf_{f \in {\cal C}_k}\dtv(f, p)$.

In this terminology, our learning problem is exactly the problem of agnostically learning the class of $k$-flat distributions.
Our main positive result is a near-optimal algorithm for this problem, i.e., a semi-agnostic learning algorithm that has near-optimal sample size 
and near-linear running time.  More precisely, we prove the following:

\begin{theorem} [Main] \label{thm:main}
There is an algorithm $A$ with the following property:  Given $k \geq 1$,
$\eps > 0$, and sample access to a target distribution $p$,  algorithm $A$ 
uses $\tilde{O}(k/\eps^2)$ independent draws from $p$, runs in time 
$\tilde{O}(k/\eps^2)$, and outputs a $O(k \log^2 (1 /\eps))$-flat hypothesis distribution $h$
that satisfies $\dtv(h, p) \leq O(\opt_k(p)) + \eps$ with probability at least $9/10$.
\end{theorem}

Using standard techniques, the confidence probability can be boosted to $1-\delta$, for any $\delta>0$, with a (necessary)
overhead of $O(\log(1/\delta))$ in the sample size and the running time. 

We emphasize that the difficulty
of our result lies in the fact that the ``optimal'' piecewise constant decomposition of the domain is both {\em unknown} and {\em approximate}
(in the sense that $\opt_k(p)>0$); and that our algorithm is both sample-optimal and runs in (near-) {\em linear time}.  Even in the
(significantly easier) case that the target $p \in {\cal C}_k$ (i.e., $\opt_k(p) = 0)$, 
and the optimal partition is explicitly given to the algorithm, it is known that a sample of size $\Omega(k/\eps^2)$ is
 information-theoretically necessary.
(This lower bound can, e.g., be deduced from the standard fact that  learning an unknown discrete distribution over a
$k$-element set to statistical distance $\eps$ requires an $\Omega(k/\eps^2)$ size sample.)
Hence, our algorithm has provably optimal sample complexity (up to a logarithmic factor), runs in essentially sample linear time, and
is $\alpha$-agnostic for a universal constant $\alpha >1$.

It should be noted that the sample size required for our problem is well-understood; it follows from the VC theorem (Theorem~\ref{thm:vc-inequality})
that $O(k/\eps^2)$ draws from $p$ are information-theoretically sufficient. However, the theorem is non-constructive, and the
``obvious'' algorithm following from it has running time 
exponential in $k$ and $1/\eps$. In recent work, Chan {\em et al}~\cite{CDSS14a} presented an approach employing an intricate combination of dynamic programming and linear programming which yields a $\poly(k/\eps)$ time algorithm for the above problem.  However, the running time of the \cite{CDSS14a} algorithm is $\Omega(k^3)$ even for constant values of $\eps$, making it impractical for applications.
As discussed below our algorithmic approach is significantly different from that of ~\cite{CDSS14a}, using neither dynamic nor linear
programming.


\noindent {\bf Applications.} 
Nonparametric density estimation for shape restricted classes has been a subject of study in
statistics since the 1950's (see \cite{BBBB:72} for an early book on the topic and~\cite{Grenander:56, Brunk:58, PrakasaRao:69, Wegman:70, HansonP:76,
Groeneboom:85, Birge:87} for some of the early literature),
and has applications to a range of areas including reliability theory {(see~\cite{Reb05aos} and references therein)}.
By using the structural approximation results of Chan {\em et al}~\cite{CDSS13}, as an immediate corollary of Theorem~\ref{thm:main} we obtain sample optimal and {\em near-linear time} estimators for various
well-studied classes of shape restricted densities including monotone, unimodal, and multimodal densities (with unknown mode locations),
monotone hazard rate (MHR) distributions, and others (because of space constraints we do not enumerate the
exact descriptions of these classes or statements of these results here, but instead refer the interested reader to \cite{CDSS13}).
Birg{\'e}~\cite{Birge:87} obtained a sample optimal and linear time estimator 
for monotone densities, but prior to our work, no linear time and sample optimal estimator was known for any of the other classes.


Our algorithm from Theorem~\ref{thm:main} is $\alpha$-agnostic for a constant $\alpha>1$.
It is natural to ask whether a significantly stronger accuracy guarantee is efficiently achievable;
in particular, is there an agnostic algorithm with similar running time and sample complexity and $\alpha = 1$?  
Perhaps surprisingly, we provide a negative answer to this question. Even in the simplest nontrivial case that $k=2$, 
and the target distribution is defined over a discrete domain $[N]=\{1,\dots,N\}$, any $\alpha$-agnostic algorithm
with $\alpha<2$ requires large sample size:

\begin{theorem}[Lower bound, Informal statement] \label{thm:main2-informal} 
Any $1.99$-agnostic learning algorithm for $2$-flat distributions 
over $[N]$ requires a sample of size $\Omega(\sqrt{N})$.
\end{theorem}

See Theorem \ref{thm:main2-formal} in Section \ref{sec:lower} for a precise statement.
Note that there is an exact correspondence between distributions over the discrete domain $[N]$ and
pdf's over $[0,1)$ which are piecewise constant on each interval of the
form $\left[ k/N,(k+1)/N \right)$ for $k \in \{0,1,\dots,N-1\}.$  Thus, Theorem \ref{thm:main2-informal} implies
that {\em no finite sample} algorithm can $1.99$-agnostically
learn even $2$-flat distributions over $[0,1)$. 
(See Corollary~\ref{cor:main2-formal} in Section \ref{sec:lower} for a detailed statement.)



\noindent {\bf Related work.}
A number of  techniques for density estimation have been developed
in the mathematical statistics literature,
including kernels and variants thereof, nearest neighbor estimators,
orthogonal series estimators, maximum likelihood estimators (MLE), and others
(see Chapter 2 of~\cite{Silverman:86} for a survey of existing methods).
The main focus of these methods has been on the statistical rate of convergence,
as  opposed to the running time of the corresponding estimators. We remark that the MLE
does not exist for very simple classes of distributions (e.g., unimodal distributions with an unknown mode, see e.g,~\cite{Birge:97}).
We note that the notion of agnostic learning is related to the literature on model selection and 
oracle inequalities~\cite{MP03}, however this work is of a different flavor and is not technically related to our results.

Histograms have also been studied extensively in various areas of computer science, including databases and streaming~\cite{JPK+98,GKS06,CMN98,GGI+02} under various
assumptions about the input data and the precise objective.
Recently, Indyk {\em et al}~\cite{ILR12} studied the problem of learning a $k$-flat distribution over $[N]$
{\em under the $L_2$ norm} and gave an efficient algorithm with sample complexity $O(k^2 \log (N) / \eps^4)$.
Since the $L_1$ distance is a stronger metric, Theorem~\ref{thm:main} 
implies an improved sample and time bound of $\tilde{O}(k/\eps^2)$ for their setting.


\section{Preliminaries} \label{sec:prelims}

Throughout the paper we assume that the underlying distributions  have 
Lebesgue measurable densities.  For a pdf $p: [0,1) \to \R_{+}$ and a Lebesgue measurable
subset $A \subseteq [0,1)$, i.e., $A \in {\cal L}([0,1))$,  we use $p(A)$ to denote $\int_{z \in A} p(z).$
The {\em statistical distance} or {\em total variation distance}
between two densities $p, q: [0,1) \to \R_{+}$ is $\dtv(p, q) : = \sup_{A \in {\cal L}([0,1))} |p(A) - q(A)|.$
The statistical distance satisfies the identity $\dtv(p,q) = {\frac 1 2} \|p - q\|_1$ where $\|p - q\|_1$, the $L_1$ distance between $p$ and $q$, 
is  $\int_{[0,1)} |p(x)-q(x)|dx$; for convenience in the rest of the paper we work with $L_1$ distance.
We refer to a nonnegative function $p$ over an
interval (which need not necessarily integrate to one
over the interval) as a ``sub-distribution.''
Given a value $\kappa> 0$, we say that a (sub-)distribution $p$ 
over $[0,1)$ is \emph{$\kappa$-well-behaved}
if $\sup_{x \in [0,1)} \Pr_{x \sim p}[x] \leq \kappa$, i.e.,
no individual real value is assigned more than $\kappa$ probability under $p$.
Any probability distribution with no atoms is $\kappa$-well-behaved for 
all $\kappa>0$.
Our results apply for general distributions over
$[0,1)$ which may have an atomic part as well as a non-atomic part.
Given $m$ independent draws $s_1,\dots,s_m$ from  a distribution $p$ over $[0,1)$,
the {\em empirical distribution} $\wh{p}_m$ over $[0,1)$
is the discrete distribution supported on $\{s_1,\dots,s_m\}$
defined as follows: for all $z \in [0,1)$,
$\Pr_{x \sim \wh{p}_m}[x=z] = |\{j \in [m] \mid s_j=z\}| / m$.

\smallskip

\noindent {\bf The VC inequality.}
Let $p:[0,1) \to \R$ be a Lebesgue measurable function.
Given a family of subsets $\mathcal A \subseteq {\cal L}([0,1))$ over $[0,1)$,
define $\norm p_{\mathcal A}
= \sup_{A\in \mathcal A} |p(A)|$.
The \emph{VC dimension} of $\mathcal A$ is the maximum size of a subset
$X\subseteq [0,1)$ that is shattered by $\mathcal A$ (a set $X$ is shattered by
$\mathcal A$ if for every $Y \subseteq X$,
some $A\in\mathcal A$ satisfies
$A\cap X = Y$).
If there is a shattered subset of size $s$ for all $s \in \Z_+$, then we say
that the VC dimension of ${\cal A}$ is $\infty$.
{The well-known \emph{Vapnik-Chervonenkis (VC) inequality} states
the following:}

\begin{theorem}[VC inequality, {\cite[p.31]{DL:01}}]
\label{thm:vc-inequality}
Let $p: I \to \R_+$ be a probability density function over $I \subseteq \R$ and
$\widehat{p}_m$ be the empirical distribution obtained after drawing $m$ points
from $p$.
Let $\mathcal A \subseteq 2^{I}$ be a family of subsets with VC dimension $d$.
Then
$ \E[ \norm{p - \widehat{p}_m}_{\mathcal A}] \leq O(\sqrt{d/m}) .$
\end{theorem}

\smallskip

\noindent {\bf Partitioning into intervals of approximately equal mass.}
As a basic primitive, given access to a sample drawn 
from a $\kappa$-well-behaved 
target distribution $p$ over $[0,1)$, we will need to partition 
$[0,1)$ into $\Theta(1/\kappa)$ intervals each of which has 
probability $\Theta(\kappa)$ under $p$.  There is a simple algorithm,
based on order statistics, which does this and
has the following performance guarantee (see 
Appendix A.2 of \cite{CDSS14a}):

\begin{lemma} \label{lem:part-approx-unif}
Given $\kappa \in (0,1)$ and access to points drawn 
from a $\kappa/64$-well-behaved
distribution $p$ over $[0,1)$,
the procedure {\tt Approximately-Equal-Partition} draws $O((1/\kappa)
\log(1/\kappa))$
points from $p$, runs in time $\tilde{O}(1/\kappa)$, and
with probability at least $99/100$ outputs a partition of
$[0,1)$ into $\ell=\Theta(1/\kappa)$ intervals such that $p(I_j)
\in [\kappa/2,  3 \kappa]$ for all $1 \leq j \leq \ell.$
\end{lemma}

\ignore{
The algorithm essentially just draws a sample of $m=\Theta((1/\kappa)
\log(1/\kappa))$ points from $p$, sorts them, and uses evenly spaced
order statistics from the sample to define the endpoints of the intervals
that comprise the partition.  The simple analysis establishing correctness
of the algorithm is given in 
}

\section{The algorithm and its analysis} \label{sec:algorithm}

In this section we prove our main algorithmic result, Theorem \ref{thm:main}.
Our approach has the following high-level structure:
In Section \ref{sec:mainresult} we give an algorithm
for agnostically learning a target distribution $p$ that is ``nice''
in two senses:
(i) $p$ is well-behaved (i.e., it does not have any heavy atomic elements), 
and (ii) $\opt_k(p)$ is bounded from above by the error parameter $\eps.$
In Section \ref{sec:cleanup} we give a general efficient reduction showing
how the second assumption can be removed, 
and in Section \ref{sec:heavyok} we briefly explain how the first 
assumption can be removed, thus yielding
Theorem \ref{thm:main}. 

\subsection{The main algorithm}
\label{sec:mainresult}

In this section we give our main algorithmic result, which
handles well-behaved distributions $p$ for which $\opt_k(p)$ is not too large:

\begin{theorem} \label{thm:algorithm}
There is an algorithm {\tt Learn-WB-small-opt-$k$-histogram} that
given as input $\tilde O(k/\eps^2)$ i.i.d. draws from a target distribution $p$
and a parameter $\eps > 0$, runs in time $\tilde{O}(k/\eps^2)$, and has the following performance guarantee:
If (i) $p$ is ${\frac {\eps/\log(1/\eps)}{384k}}$-well-behaved, and
(ii) $\opt_k(p) \leq \eps$, then with probability at least $19/20$, 
it outputs an $O(k \cdot \log^2 (1/\eps))$-flat distribution $h$
such that $\dtv(p,h) \leq 2 \cdot \opt_k(p)+ 3\eps$. 
\end{theorem}

We require some notation and terminology.
Let $r$ be a distribution over $[0,1)$, and let $\calP$ be a set of disjoint
intervals that are contained in $[0, 1)$. 
We say that the \emph{$\calP$-flattening of $r$}, denoted $\flatten{r}{\calP}$,
is the sub-distribution defined as
\[r(v) = \left\{
\begin{array}{l l}
r(I) / |I| & \text{if } v \in I, I \in \calP\\
0 & \text{if $v$ does not belong to any $I \in \calP$}
\end{array} \right.
\]Observe that if
$\calP$ is a partition of $[0,1)$, then (since $r$
is a distribution) $(r)^\calP$ is a distribution.

We say that two intervals $I,I'$ are \emph{consecutive}
if $I=[a,b)$ and $I'=[b,c)$.  Given two
consecutive intervals $I,I'$ contained in $[0, 1)$ and a sub-distribution
$r$, we use $\alpha_r(I,I') $ to denote the $L_1$ distance
between 
$\flatten{r}{\{I, I'\}}$  and $\flatten{r}{\{I \cup I'\}}$,
i.e., $ \alpha_r(I,I') = \int_{I \cup I'}|(r)^{\{I,I'\}}(x)
- (r)^{\{I \cup I'\}}(x)|dx.$
Note here that $\{I \cup I'\}$ is a set that contains one element,
the interval $[a,c)$.

\subsubsection{Intuition for the algorithm}

We begin with a high-level intuitive explanation of the 
{\tt Learn-WB-small-opt-$k$-histogram} algorithm.  
It starts in Step~1 by constructing
a partition of $[0,1)$ into $z=\Theta(k/\eps')$ intervals $I_1,\dots,I_z$
(where $\eps'= \tilde{\Theta}(\eps)$) such that $p$ has weight 
$\Theta(\eps'/k)$ on each subinterval.
In Step~2 the algorithm draws a sample of $\tilde{O}(k/\eps^2)$ points from $p$
and uses them to define an empirical distribution $\widehat{p}_m$.  This
is the only step in which points are drawn from $p$.  For the rest of
this intuitive explanation we pretend that the weight
$\widehat{p}(I)$ that the empirical distribution $\widehat{p}_m$ assigns
to each interval $I$ is actually the same as the true weight $p(I)$
(Lemma \ref{lem:cdss} below shows that this is not too far from the truth).

Before continuing with our explanation of the algorithm, let us 
digress briefly by 
imagining for a moment that the target distribution $p$ actually is a $k$-flat
distribution (i.e., that $\opt_k(p)=0$).  In this case there are at most
$k$ ``breakpoints'', and hence at most $k$ intervals $I_j$ for
which $\alpha_{\widehat{p}_m}(I_j,I_{j+1}) > 0$, so computing the
$\alpha_{\widehat{p}_m}(I_j,I_{j+1})$ values would be an easy
way to identify the true breakpoints (and given these it is not difficult
to construct a high-accuracy hypothesis).

In reality, we may of course have $\opt_k(p)>0$; this means that if we try to 
use the $\alpha_{\widehat{p}_m}(I_j,I_{j+1})$ criterion to identify 
``breakpoints'' of the optimal $k$-flat distribution that is closest to $p$
(call this $k$-flat distribution $q$), we may sometimes be ``fooled''
into thinking that $q$ has a breakpoint in an interval $I_j$ where it
does not (but rather the value 
$\alpha_{\widehat{p}_m}(I_j,I_{j+1})$ is large because of the difference
between $q$ and $p$).  However, recall that by assumption
we have $\opt_k(p) \leq \eps$; this bound can be used to show that there 
cannot be too many intervals $I_j$ for which a large value of 
$\alpha_{\widehat{p}_m}(I_j,I_{j+1})$ suggests a ``spurious breakpoint'' 
(see the proof of Lemma \ref{lem:numintervals}).  This 
is helpful, but in and of itself not enough; since our partition 
$I_1,\dots,I_z$ divides $[0,1)$ into $k/\eps'$ intervals, a naive approach 
based on this would result in a $(k/\eps')$-flat hypothesis distribution,
which in turn would necessitate a sample complexity of 
$\tilde{O}(k/\eps'^3)$, which is unacceptably
high.  Instead, our algorithm performs a careful process of iteratively
merging consecutive intervals for which the 
$\alpha_{\widehat{p}_m}(I_j,I_{j+1})$ criterion indicates that a merge
will not adversely affect the final accuracy by too much.  As a result of this
process we end up with $k \cdot \polylog(1/\eps)$ intervals for the 
final hypothesis, which enables us to output a 
$(k \cdot \polylog(1/\eps'))$-flat final hypothesis using
$\tilde{O}(k/\eps'^2)$ draws from $p$.

In more detail, this iterative merging is carried out by 
the main loop of the algorithm in Step~4.
Going into the $t$-th iteration of the loop, the algorithm
has a partition $\calP_{t-1}$ of $[0,1)$ into disjoint sub-intervals, and a 
set $\calF_{t-1} \subseteq \calP_{t-1}$ (i.e., every interval belonging 
to $\calF_{t-1}$ also belongs to $\calP_{t-1}$).
Initially $\calP_{0}$ contains all the intervals $I_1,\dots,I_z$
and $\calF_0$ is empty.
Intuitively, the intervals in $\calP_{t-1} \setminus \calF_{t-1}$ 
are still being ``processed'';
such an interval may possibly be merged with a consecutive interval
from $\calP_{t-1} \setminus \calF_{t-1}$ 
if doing so would only incur a small ``cost'' (see condition
(iii) of Step 4(b) of the algorithm).\ignore{
}The intervals in $\calF_{t-1}$ have been ``frozen'' and will not be
altered or used subsequently in the algorithm.

\subsubsection{The algorithm}
\begin{framed}
\noindent {\bf Algorithm 
{\tt Learn-WB-small-opt-$k$-histogram}:
}

\smallskip

\noindent {\bf Input:}  parameters $k \geq 1,\eps > 0$;
access to i.i.d. draws from target
distribution $p$  over $[0,1)$

\noindent {\bf Output:}  
If (i) $p$ is ${\frac {\eps/\log(1/\eps)}{384k}}$-well-behaved
and (ii) $\opt_k(p) \leq \eps$,
then with probability at least $99/100$ the output is a 
distribution $q$ such that $\dtv(p,q) \leq 2 \opt_k(p) + 3 \eps.$

\begin{enumerate}

\item Let $\eps' = \eps / \log (1/\eps)$.
Run Algorithm~{\tt Approximately-Equal-Partition} on input
parameter ${\frac {\eps'}{6k}}$ to partition $[0,1)$ into
$z = \Theta(k/\eps')$ intervals $I_1 = [i_0,i_1)$,
$\dots,$ $I_{z}=[i_{z-1},i_z)$, where $i_0=0$ and $i_z=1$,
\ignore{{\tiny{\red{[[[Rocco:  had ``and $z$ is even'' here, but I think we do
not need it]]]}}}} 
such that
with probability at least $99/100$, for each $j \in \{1,\dots,z\}$ 
we have $p([i_{j-1},i_j)) \in [\eps'/12k,\eps'/2k]$ (assuming
$p$ is $\eps'/(384k)$-well-behaved).

\item Draw $m=\tilde{O}(k/\eps'^2)$ points from $p$ 
and let $\widehat{p}_m$ be the resulting empirical distribution.

\item Set $\calP_0 = \{I_1, I_2, \dots I_{z}\}$, and $\calF_0 = \emptyset$.

\item Let 
$s = \log_2 {\frac{1}{\eps '}}$.
Repeat for $t=1,\dots$ until $t=s$:

\begin{enumerate}

\item Initialize $\calP_t$ to $\emptyset$ and $\calF_t$ to  $\calF_{t-1}$.

\item 
Without loss of generality, assume $\calP_{t-1} = \{I_{t-1, 1}, \dots, 
I_{t-1, z_{t-1}} \}$ where interval $I_{t-1,i}$ is to the left 
of $I_{t-1,i+1}$ for all $i$.
Scan left to right across the intervals in $\calP_{t - 1}$ (i.e., iterate over
$i=1,\dots,z_{t-1}-1$).
If intervals  $I_{t-1, i},I_{t-1, i+1}$ are
\ignore{
{\tiny{\red{[[[Rocco: got rid of ``consecutive'' --- since $\calP_{t-1}$ is a 
disjoint partition of $[0,1)$, it will necessarily be the case
tht these two intervals are consecutive]]]}}}
}
(i) both not in $\calF_{t - 1}$, and 
(ii) $\alpha_{\widehat{p}_m}(I_{t-1, i},I_{t-1, i+1})> \eps'/(2k)$,  then add 
both $I_{t-1, i}$ and $I_{t-1, i+1}$ into $\calF_t$. 

\item Initialize $i$ to 1, and repeatedly execute one of the
following four (mutually exclusive and exhaustive) cases until $i > z_{t - 1}$:

[Case 1] $i \leq z_{t - 1} - 1$ and  $I_{t-1, i}
=[a,b), I_{t-1, i+1}=[b,c)$ are consecutive intervals both
not in $\calF_t$. Add the merged interval 
$I_{t-1, i} \cup I_{t-1, i+1} =[a,c)$ into 
$\calP_t$. Set $i \leftarrow i+2$.

[Case 2]   $i \leq z_{t - 1} - 1$ and $I_{t-1, i} \in \calF_t$. Set $i \leftarrow i+1$.

[Case 3]  $i \leq z_{t - 1} - 1$, $I_{t-1, i} \notin \calF_t$ and $I_{t-1, i + 1} \in \calF_t$. Add $I_{t-1, i}$ into $\calF_t$ and set $i \leftarrow i+2$.

[Case 4] 
$i=z_{t-1}$.  Add $I_{t-1,  z_{t - 1}}$ into $\calF_t$ if $I_{t-1,  z_{t - 1}}$ 
is not in $\calF_t$ and set $i \leftarrow i + 1$.
\ignore{
{\tiny{\blue{[[[Rocco:  this was ``$i = z_{t - 1}$.  Insert $I_{t-1,  z_{t - 1} - 1}$ into $\calF_t$ if $I_{t-1,  z_{t - 1} - 1}$ is not in $\calF_t$ and set $i \leftarrow i + 1$.''  I think this case is supposed to capture the possibility
that $i$ is $z_{t-1},$ the very last interval, and in this case
we just put it in $\calF_t$ since it cannot be paired up with a subsequent
consecutive interval.  It seemed weird to me that
in this case it was going back to look again at $z_{t-1}-1.$
Let me know if what I changed it to is correct.}}}
}

\item Set $\calP_t \leftarrow \calP_t \cup \calF_t$.
\end{enumerate}
\item Output the $|\calP_s|$-flat hypothesis distribution $(\widehat p_m)^{\calP_{s}}$.
\end{enumerate}
\end{framed}

\subsubsection{Analysis of the algorithm and proof of Theorem
\ref{thm:algorithm}}

It is straightforward to verify the claimed running time given
Lemma \ref{lem:part-approx-unif}, which bounds the running time of
{\tt Approximately-Equal-Partition}.  Indeed, we note that Step~2,
which simply draws $\tilde{O}(k/\eps'^2)$ points and constructs
the resulting empirical distribution, dominates the overall running time.
\ignore{This is because  
Step~1 (running {\tt Approximately-Equal-Partition}) takes time
$\tilde{O}(k/\eps')$, and it is straightforward to implement
Step~4 to run in time $\tilde{O}(k/\eps')$ as well (note that the outer
loop over $t$ only goes for $\log(1/\eps')$ repetitions,
and that for each $t$ the number of intervals in $\calP_t$ is at most
$O(k/\eps')$).
}  
In the rest of this subsection we prove correctness.

We first observe that with high probability the empirical
distribution $\widehat{p}_m$ defined in Step~2 
gives a high-accuracy estimate of
the true probability of any union of consecutive intervals from
$I_1,\dots,I_z$.  The following lemma from \cite{CDSS14a} follows from the standard multiplicative 
Chernoff bound:
\begin{lemma}[{Lemma 12, \cite{CDSS14a}}]\label{lem:cdss}
With probability $99/100$ over the sample drawn in Step~2, 
for every $0 \leq a < b \leq z$ we have that
$|\widehat{p}_m([i_a, i_b)) -  p([i_a, i_b))| 
\leq \sqrt{\eps' (b - a)}  \cdot \eps'/ (10 k).$
\end{lemma}

We henceforth assume that this $99/100$-likely event indeed takes place,
so the above inequality holds for all $0 \leq a< b \leq z.$
We use this to show that the 
$\alpha_{\widehat{p}_m}(I_{t-1,i}, I_{t-1,i+1})$ value that the algorithm
uses in Step 4(b) is a good proxy for the actual value 
$\alpha_{p}(I_{t-1,i}, I_{t-1,i+1})$ (which of course is not accessible
to the algorithm):

\begin{lemma}\label{lem:difference}
Fix $1 \leq t \leq s.$  Then we have
$
|\alpha_{\hat{p}_m}(I_{t-1,i},I_{t-1,i+1}) -
 \alpha_{p}(I_{t-1,i},I_{t-1,i+1})| \leq 2\eps'/(5k).$
\end{lemma}
\begin{proof}
Observe that in iteration $t$, two consecutive intervals 
$I_{t-1,i}$ and $I_{t-1,i+1}$ correspond to two unions of consecutive
intervals $I_a \cup \cdots \cup I_b$ and $I_{b+1} \cup \cdots \cup I_c$
respectively from the original partition $\calP_0$.  Moreover, since each
interval in $\calP_{t-1} \setminus \calF_{t-1}$, $t>1$, is formed by merging two
consecutive intervals from $\calP_{t-2} \setminus \calF_{t-2}$, 
it must be the case that $b-a+1,c-b+1 \leq
2^{t-1} < 2^{s-1} \leq 1/(2\eps')$.  Hence,
by Lemma \ref{lem:cdss}, we have 
\begin{equation*}
|p(I_{t-1, i} ) - \widehat{p}_m(I_{t-1, i}))| \leq \sqrt{\eps'  \cdot 2^{s-1}}
\cdot \frac{\eps'}{10 k} \leq \frac{\eps'}{10\sqrt{2}k}
\end{equation*}
and similarly,
\begin{equation*}
|p(I_{t-1, i+1}) - \widehat{p}_m(I_{t-1, i + 1}))| \leq \frac{\eps'}{10\sqrt{2}k}.
\end{equation*}

%

To simplify notation, let $I = I_{t-1,i}$ and $J = I_{t-1,i+1}$.
By definition of $\alpha$,
\begin{eqnarray} \label{eq:alpha}
\alpha_p(I,J) &=& \left\lvert\frac{p(I)}{|I|} -
\frac{p(I)+p(J)}{|I|+|J|}\right\rvert |I| + \left\lvert\frac{p(J)}{|J|} -
\frac{p(I)+p(J)}{|I|+|J|}\right\rvert |J| \nonumber \\
&=&\frac{2}{|I|+|J|} \big\lvert p(I)|J| - p(J)|I| \big\rvert .
\end{eqnarray}
A straightforward calculation now gives that 
\begin{eqnarray*}
 |\alpha_p(I, J) - \alpha_{\hat p_m}(I, J)|
 &=& \frac2{|I|+|J|} \Big\lvert \big\lvert p(I)|J| - p(J)|I|\big\rvert -
 \big\lvert \hat p_m(I)|J| - \hat p_m(J)|I| \big\rvert \Big\rvert \\
 &\leq& \frac2{|I|+|J|} \Big( \big\lvert p(I) - \hat p_m(I) \big\rvert |J| +
 \big\lvert p(J) - \hat p_m(J) \big\rvert |I| \Big) \\
&\leq& 2\eps'/(5k).
\end{eqnarray*}
\end{proof}

For the rest of the analysis,
let $q$ denote a fixed $k$-flat distribution that is closest to $p$, so
$\|p-q\|_1 = \opt_k(p)$.  (We note that while
$\opt_k(p)$ is defined as $\inf_{q \in {\cal C}} \|p-q\|_1$, 
standard closure arguments can be used to show that the infimum is
actually achieved by some $k$-flat distribution $q$.)
Let $\calQ$ be the partition of $[0,1)$ corresponding to the intervals
on which $q$ is piecewise constant.
We say that a \emph{breakpoint} of $\calQ$ is a value in $[0,1]$
that is an endpoint of one of the (at most) $k$ intervals in $\calQ$.

The following important lemma bounds the number of intervals in the
final partition $\calP_s$:
\begin{lemma}
\label{lem:numintervals}
$\calP_s$ contains at most $O(k \log^2 (1/\eps))$ intervals.
\end{lemma}
\begin{proof}
We start by recording a basic fact that will be useful in the proof of the lemma.
Let $p$ be a distribution over an interval $I$ 
and let $q$ be any sub-distribution over $I$. Perhaps contrary to initial 
intuition, the optimal scaling $c \cdot q$, $c > 0$, of $q$ to approximate $p$
(with respect to the $L_1$-distance) is not necessarily  obtained by 
scaling $q$ so that $c \cdot q$ is a distribution over $I$.
However, a simple argument (see e.g.,  Appendix A.1 of \cite{CDSS14a}) shows that scaling so that $c \cdot q$ 
is a distribution cannot result in $L_1$-error more than twice that of the 
optimal scaling:

\begin{claim} \label{claim:simple}
Let $p,g: I \to \R^{\geq 0}$ be probability  distributions over $I$ (so $\int_I p(x) dx= \int_I g(x) dx = 1$).  Then, writing $\|f\|_1$ to denote $\int_I |f(x)| dx$,
for every $a > 0$ we have that $\|p-g\|_1\leq 2 \|p - a g\|_1.$
\end{claim}

We now proceed with the proof of Lemma~\ref{lem:numintervals}.

We first show that a total of at most $O(k \log(1/\eps'))$ intervals are
ever added into $\calF_t$ across all executions of Step 4(b).

Suppose that intervals $I_{t-1, i}, I_{t-1, i+1}$ are added into $\calF_t$ 
in some execution of Step 4(b). We consider the following two cases:

\begin{itemize}
\item[{\bf Case 1:}] 
$I_{t-1, i} \cup I_{t-1, i+1}$ contains at least one 
breakpoint of $\calQ$. Since $\calQ$ has at most $k$ breakpoints, this can 
happen at most $k$ times in total.

\item[{\bf Case 2:}] 
$I_{t-1, i} \cup I_{t-1, i+1}$ does not contain any breakpoint of $\calQ$. 
Then $I_{t-1, i} \cup I_{t-1, i+1}$ is a subset of an interval in $\calQ$.
Recalling that intervals $I_{t-1,i},I_{t-1,i+1}$ were added into ${\cal F}_t$ in an execution of Step 4(b), 
we have that $\alpha_{\hat{p}_m}(I_{t-1,i},I_{t-1,i+1}) > \eps'/(2k),$ and hence
by Lemma \ref{lem:difference}, we have that
$\alpha_{p}(I_{t-1,i},I_{t-1,i+1}) \geq {\frac 1 5} \cdot {\frac {\eps'} k}$.
Claim \ref{claim:simple} now implies that
the contribution to the $L_1$ distance between $p$ and $q$
from $I_{t-1, i} \cup I_{t-1, i+1}$,
i.e., $\int_{I_{t-1,i} \cup I_{t-1,i+1}} |p(x)-q(x)|dx$,
is at least $\frac{1}{10}\frac{\eps'}{k}$.

Since $\|p-q\|_1 = \opt_k(p),$ there can be at most 
\[
k+O\left(\frac{\opt_k(p) \cdot k}{\eps'}\right)
= O\left(k \cdot \log {\frac 1 \eps}\right)
\]
intervals ever added into $\calF_t$ across all executions of Step 4(b)
(note that for the last equality we have used the assumption
that $\opt_k(p) \leq \eps$).
\end{itemize}

Next, we argue that each $\calF_t$ satisfies
$|\calF_t| \leq O(k \log^2(1/\eps)).$
We have bounded the number of intervals added into $\calF_t$ in Step
4(b) by $O(k \log(1/\eps'))$, so it remains to bound the number of
intervals added in Step 4(c)(Case 3) and 4(c)(Case 4).  It is clear
that a total of at most $O(\log(1/\eps'))$ intervals are ever added in
4(c)(Case 4).  Inspection of Step 4(c)(Case 3) shows that for a given
value of $t$, the number of intervals that this step adds to $\calF_t$
is at most the number of ``blocks'' of consecutive $\calF_t$-intervals.  
Since each interval added in Step 4(c)(Case 3) extends some 
blocks of consecutive $\calF_t$-intervals but does not create a new one
(and hence does not increase their number),
across the $s=\log(1/\eps')$ stages, the total number of intervals
that can be added in executions of Step 4(c)(Case 3) 
is at most $O(k \log^2(1/\eps'))$.
It follows that we have $|\calF_s| = O(k \log^2(1/\eps))$ as claimed.

To bound $|\calP_{t} \setminus \calF_{t}|$, we observe that by inspection
of the algorithm, for each $t$ we have 
$|\calP_{t} \setminus \calF_{t}|\leq {\frac 1 2}
|\calP_{t-1} \setminus \calF_{t-1}|.$
Since $|\calP_0| = \Theta(k/\eps')$, it follows that $|\calP_s \setminus \calF_s| = O(k)$,
and the lemma is proved.
\end{proof}

The following definition will be useful:

\begin{definition} \label{def:goodpartition}
Let $\calP$ denote any partition of $[0, 1)$. 
We say that partition $\calP$ is \emph{$\eps'$-good} for $(p, q)$
if for every breakpoint $v$ of $\calQ$, the interval $I$ in $\calP$ 
containing $v$ satisfies $p(I) \leq \eps' / (2k).$
\end{definition}

The above definition is justified by the following lemma:

\begin{lemma}\label{lem:distance1}
If $\calP$ is $\eps'$-good for $(p, q)$, 
then 
$\|p - \flatten{p}{\calP}\|_1 \leq 2\opt_k(p) + \eps'$. 
\end{lemma}
\begin{proof}
Fix an interval $I$ in $\calP$.
If there does not exist an interval $J$ in $\calQ$ such that $I \subseteq J$, then
$I$ must contain a breakpoint of $\calQ$, and hence since $\calP$ is $\eps'$-good
for $(p,q)$, we have $p(I)  \leq \eps'/(2k)$.  This implies that the
contribution to $\|\flatten{p}{\calP} - q\|_1$ that comes from $I$, namely
$\int_{I} |\flatten{p}{\calP}(x)-q(x)|dx$, satisfies
\begin{eqnarray*}
\int_{I} |\flatten{p}{\calP}(x)-q(x)|dx &\leq&
\int_{I} |\flatten{p}{\calP}(x)-p(x)|dx + 
\int_{I} |p(x)-q(x)|dx\\
&\leq&
\int_{I} |p(x)-q(x)|dx + 2p(I)\\
&\leq&
\int_{I} |p(x)-q(x)|dx + {\frac {\eps'}{k}}.
\end{eqnarray*}

The other possibility is that there exists an interval $J$ in $\calQ$ such  
that $I \subseteq J$. In this case, we have that
\[
\int_{I} |\flatten{p}{\calP}(x)-q(x)| dx \leq
\int_{I} |p(x)-q(x)| dx.
\]

Since there are at most $k$ intervals in $\calP$ containing breakpoints of $\calQ$,
summing the above inequalities over all intervals $I$ in $\calP$, we get that
\[
\|\flatten{p}{\calP} - q\|_1 \leq
\|p - q\|_1 + \eps' = \opt_k(p)+\eps',
\]
and hence
\[
\|\flatten{p}{\calP} - p\|_1 \leq
\|\flatten{p}{\calP} - q\|_1 + \|p - q\|_1 \leq 
2\opt_k(p)+\eps'.
\]
\end{proof}

We are now in a position to prove the following:

\begin{lemma}\label{lem:distance2}
There exists a partition $\calR$ of $[0,1)$ that is $\eps'$-good for $(p, q)$ 
and satisfies
\[
\|\flatten{p}{\calP_s} - \flatten{p}{\calR}\|_1 \leq \eps.\]
\end{lemma}
\begin{proof}
We construct the claimed $\calR$ based on $\calP_s,\calP_{s-1},\dots,\calP_0$ 
as follows:
\begin{itemize}
\item [(i)] If $I$ is an interval in $\calP_s$ not containing a breakpoint of $\calQ$, 
then $I$ is also in $\calR$. 
\item [(ii)] 
If $I$ is an interval in $\calP_s$ that does contain a breakpoint of $\calQ$,
then we further partition $I$ into a set of intervals $S$ by calling procedure {\tt Refine-partition}($s, I$).  This recursive procedure exploits the
local structure of the earlier, finer partitions $\calP_{s-1},\calP_{s-2},\dots$
as described below. 

\begin{framed}
\noindent {\bf Procedure {\tt Refine-partition}:}

\smallskip

\noindent {\bf Input:}   Integer $t$, Interval $J$

\noindent {\bf Output:}  $S$, a partition of interval $J$

\begin{enumerate}
\item If $t = 0$, 
 then output $\{J\}$.
\item If $J$ is an interval in $\calP_t$, then 
\begin{enumerate}
\item If $J$ contains a breakpoint of $\calQ$, then output   {\tt Refine-partition}($t-1$, $J$).
\item Otherwise output $\{J\}$.
\end{enumerate}
\item Otherwise, $J$ is a union of two intervals in $\calP_t$. Let $J_1$ and $J_2$ denote the two intervals in $\calP_t$ such that $J_1 \cup J_2 = J$. Output {\tt Refine-partition}($t$, $J_1 $) $\cup$ {\tt Refine-partition}($t$, $J_2 $).
\end{enumerate}
\end{framed}
\end{itemize}

We claim that $|\calR|$ (the number of intervals in $\calR$) is at most
$|\calP_s| + O(k \cdot \log \frac{1}{\eps})$.
To see this, note that each interval $I \in \calP_s$ not containing
a breakpoint of $\calQ$ (corresponding to (i) above) translates directly to
a single interval of $\calR$.  For each interval of type (ii) in $\calP_s$,
inspection of the {\tt Refine-Partition} procedure shows that 
that these intervals are partitioned into at most $O(k \log (1/\eps))$ 
intervals in $\calR$.

In the rest of the proof, we show that for any interval $J$ in $\calP_s$ 
containing at least one breakpoint of $\calQ$, 
the contribution to the $L_1$ distance between 
$\flatten{p}{\calP_s}$ and $\flatten{p}{\calR}$ coming
from interval $J$ is at most $|b_J| \cdot \frac{\eps' \log \frac{1}{\eps}}{k}$, where $b_J$ is the set of breakpoints of $\calQ$ in $J$.

Consider a fixed breakpoint $v$ of $\calQ$. Let $I_{t, v}$ denote the 
interval containing $v$ in the partition $\calP_t$. If $I_{t, v}$ 
merges with another interval in $\calP_t$ in Case 1 of Step 4(c), 
we denote that other interval as $I_{t, v}'$. 
Since $I_{t,v}$ merges with $I'_{t,v}$ in Case 1 of Step 4(c), these intervals
are both not in ${\cal F}_t$ and hence were both not in 
${\cal F}_{t-1}$ in Step 4(b).
Consequently when $t > 1$ it must be the case that condition (ii) of Step 4(b)
does not hold for these intervals, i.e., $\alpha_{\hat{p}_m}(I_{t,v},I'_{t,v})
\leq \eps'/(2k).$  It follows that by Lemma \ref{lem:difference}, we have that
$\alpha_p(I_{t,v}, I_{t, v}')$ is at most $\frac{4\eps'}{5k}$.
When $t = 1$, we have a similar bound $\alpha_p(I_{t,v},I'_{t,v}) \leq
\eps'/k$, by using \eqref{eq:alpha} and the fact that $p(I_{t,v}),p(I'_{t,v})
\leq \eps'/2k$ when $I_{t,v},I'_{t,v} \in \calP_0$.


On the other hand, inspection of the procedure {\tt Refine-Partition}
gives that if two intervals in $\calP_t$ are unions of some intervals in {\tt Refine-partition}$(s, I)$, and their union is an interval in
$\calP_{t+1}$, then there exists $v$ which is a breakpoint of $\calQ$ such that the two intervals are $I_{t, v}$ and $I_{t, v}'$.

Thus, the contribution to the $L_1$ distance 
between $\flatten{p}{\calP_s}$ and $\flatten{p}{\calR}$ coming from interval $J$
is at most $\frac{\eps'}{k} \cdot \log \frac{1}{\eps'} \cdot |b_J|$.
Summing over all intervals $J$ that contain at least one breakpoint
and recalling that the total number of breakpoints
is at most $k$, we get that the overall $L_1$ distance
between $\flatten{p}{\calP_s}$ and $\flatten{p}{\calR}$ is at most $\eps$.
\end{proof}

Finally, by putting everything together we can prove 
Theorem~\ref{thm:algorithm}:

\begin{proof}[{\bf \em Proof of Theorem \ref{thm:algorithm}}]
By Lemma \ref{lem:distance1} applied to $\calR$, we have that
$\|p - \flatten{p}{\calR}\|_1 \leq 2 \opt_k(p)+\eps'.$
By Lemma \ref{lem:distance2}, we have that 
$\|\flatten{p}{\calP_s} - \flatten{p}{\calR}\|_1 \leq \eps$; thus the triangle
inequality gives that
$\|p - \flatten{p}{\calP_s}\|_1 \leq 2 \opt_k(p) + 2 \eps.$
By Lemma \ref{lem:numintervals} the partition $\calP_s$ contains at most
$O(k \log^2(1/\eps))$ intervals, so both $\flatten{p}{\calP_s}$ and
$\flatten{\widehat{p}_m}{\calP_s}$ are $O(k \log^2(1/\eps))$-flat distributions.
Thus, 
$\|\flatten{p}{\calP_s} - \flatten{\widehat{p}_m}{\calP_s}\|_1 = 
\|\flatten{p}{\calP_s} - \flatten{\widehat{p}_m}{\calP_s}\|_{{\cal A}_\ell}$,
where $\ell = O(k \log^2(1/\eps))$ 
and ${\cal A}_{\ell}$ is the family of all subsets of $[0,1)$ that consist of
unions of up to $\ell$ intervals (which has VC dimension $2 \ell$).
Consequently by the VC inequality (Theorem \ref{thm:vc-inequality},
for a suitable choice of $m=\tilde{O}(k/\eps'^2)$, we have that
$\E[
\|\flatten{p}{\calP_s} - \flatten{\widehat{p}_m}{\calP_s}\|_1] \leq 4\eps'/100.
$  Markov's inequality now gives that with probability at least $96/100$,
we have
$\|\flatten{p}{\calP_s} - \flatten{\widehat{p}_m}{\calP_s}\|_1 \leq \eps'.$
Hence, with overall probability at least $19/20$ (recall 
the 1/100 error probability incurred in Lemma \ref{lem:cdss}),
we have that 
$\|p - \flatten{\widehat{p}_m}{\calP_s}\|_1 \leq 
2 \opt_k(p) + 3 \eps,$ and the theorem is proved.
\end{proof}

\subsection{A general reduction to the case of small $\opt$ for semi-agnostic 
learning} \label{sec:cleanup}

In this section we show that under mild conditions, the general problem 
of agnostic distribution learning for a class ${\cal C}$
can be efficiently reduced to the special case when $\opt_{\cal C}$ is not
too large compared with $\eps$.  While the reduction is simple
and generic, we have not previously encountered it in the literature
on density estimation, so we provide a proof in the following.
A precise statement of the reduction follows:

\begin{theorem} \label{thm:reduction}
Let $A$ be an algorithm with the following behavior:  $A$
is given as input i.i.d. points drawn from $p$ and a parameter $\eps > 0$.
$A$ uses $m(\eps) = \Omega(1/\eps)$ draws from $p$, runs in time 
$t(\eps) = \Omega(1/\eps)$,
and satisfies the following:  if $\opt_{{\cal C}}(p) \leq 
10\eps$, then with probability at least $19/20$ it outputs a hypothesis distribution
$q$ such that 
(i) $\|p - q\|_1 \leq \alpha \cdot \opt_{{\cal C}}(p) + \eps$,
where $\alpha$ is an absolute constant, and (ii) given any $r \in [0,1)$, the value
$q(r)$ of the pdf of $q$ at $r$ can be efficiently computed in $T$ time steps.

Then there is an algorithm $A'$ with the following performance guarantee:
$A'$ is given as input i.i.d. draws from $p$ and a parameter 
$\eps > 0$.\footnote{
Note that now there is no guarantee that $\opt_{{\cal C}}(p) \leq \eps$; indeed,
the point here is that $\opt_{{\cal C}}(p)$ may be arbitrary.}
Algorithm $A'$ uses $O(m(\eps/10) 
+ \log \log(1/\eps)/\eps^2)$ draws from $p$,
runs in time $O(t(\eps/10)) + T \cdot \tilde{O}(1/\eps^2)$, 
and outputs a hypothesis distribution
$q'$ such that with probability at least $39/40$ we have
$\|p - q'\|_1 \leq 10(\alpha+2)\cdot \opt_{{\cal C}}(p) + \eps.$
\end{theorem}
\begin{proof}
The algorithm $A'$ works in two stages, which we describe
and analyze below.

In the first stage, $A'$ iterates over $\lceil \log(20/\eps) \rceil$ 
``guesses''
for the value of $\opt_{{\cal C}}(p)$, where the $i$-th guess $g_i$
is 
${\frac \eps {10}} \cdot 2^{i-1}$ (so 
$g_1={\frac \eps {10}}$ and $g_{\lceil \log(20/\eps) \rceil} \geq 1$).
For each value of $g_i$, it performs $r=O(1)$ runs of Algorithm $A$
(using a fresh sample from $p$ for each run) using parameter $g_i$
as the ``$\eps$'' parameter for each run; let $h_{1,i},\dots,h_{r,i}$
be the $r$ hypotheses thus obtained for the $i$-th guess.
It is clear that this stage uses $O(m(\eps/10) + m(2 \eps/10) + 
\cdots) = O(m(\eps))$ draws from $p$,
and similarly that it runs
in time $O(t(\eps))$.  If $\opt_{{\cal C}}(p) \leq \eps$, then (for a suitable choice of
$r=O(1)$) we get that with probability at least 39/40, some hypothesis $h_{1,\ell}$ satisfies
$\|p - h_{1,\ell}\| \leq \alpha \cdot \opt_{{\cal C}}(p) + \eps/10$.  Otherwise,
 there must be some $i \in \{2,\dots,\lceil \log(20/\eps) \rceil\}$ 
such that
$g_i/2 < \opt_{{\cal C}}(p) \leq g_i$; in this case,
for a suitable choice of $r=O(1)$ we get that with probability at least
$39/40$, there is some hypothesis $h_{i,\ell}$ that satisfies
$\|p - h_{i,\ell}\|_1 \leq \alpha \cdot  \opt_{{\cal C}}(p) + g_i \leq
(\alpha + 2) \cdot \opt_{{\cal C}}(p)$.  Thus in either event, with probability at least
$39/40$ some $h_{i,\ell}$ satisfies $\|p - h_{i,\ell}\|_1 \leq 
(\alpha + 2) \cdot \opt_{{\cal C}}(p)+ \eps/10.$

In the second stage, $A'$ runs a hypothesis selection procedure to choose
one of the candidate hypotheses $h_{i,\ell}$. A number of such procedures are
known (see e.g., Section 6.6 of \cite{DL:01} or
\cite{DDS13iaug,DaskalakisKamath14,AJOS14}); all of them work by running some sort of ``tournament'' over  the hypotheses, and all have the guarantee that with high probability
they will output a hypothesis from the pool of candidates which
has $L_1$ error (with respect to the target distribution $p$)
not much worse than that of the best candidate in the pool.
We use the classic Scheff\'e algorithm (see \cite{DL:01}) as described and analyzed in
\cite{AJOS14} (see Algorithm SCHEFFE$^*$ in Appendix B of that paper).
Adapted to our context, this algorithm has the following performance guarantee:



\begin{proposition} \label{prop:scheffe-ajos14}
Let $p$ be a target distribution over $[0,1)$
and let ${\cal D}_\tau = \{ p_j\}_{j=1}^N$ be a collection of $N$ distributions
over $[0,1)$ with the property that there exists $i \in [N]$ such that
$\|p-p_i\|_1 \leq \tau$.
There is a procedure SCHEFFE which is given as input a parameter $\eps > 0$ and
a confidence parameter $\delta>0$, and is provided
with access to
\begin{itemize}

\item [(i)] i.i.d. draws from $p$ and from $p_i$ for all $i \in [N]$, and

\item [(ii)] an \emph{evaluation oracle} $eval_{p_i}$
for each  $ \in [N]$.  This is a procedure which, on input $r \in [0,1)$,
outputs the value $p_i(r)$ of the pdf of $p_i$ at the point $r$.

\end{itemize}

The procedure SCHEFFE has the following behavior:
It makes $s = O\left( (1/ \eps^{2}) \cdot (\log N + \log(1/\delta)) \right)$ draws from $p$ and
from each $p_i$, $i \in [N]$,
and $O(s)$ calls to each oracle $eval_{p_i}$, $i \in [N]$, and performs $O(s N^2)$ arithmetic operations.
With probability at least $1-\delta$ it outputs an index $i^{\star} \in [N]$ that satisfies $\|p - p_{i^\star}\|_1 \leq 10 \max\{\tau,\eps\}.$
\end{proposition}

The algorithm $A'$ runs the procedure SCHEFFE using the $N=O(\log(1/\eps))$ hypotheses $h_{i,\ell}$,
with its ``$\eps$'' parameter set to ${\frac 1 {10} \cdot}($the 
input parameter $\eps$ that is given to $A')$ and its ``$\delta$'' parameter set to $1/40$.  By Proposition \ref{prop:scheffe-ajos14}, with overall
probability at least $19/20$ the output is a hypothesis $h_{i,\ell}$ satisfying
$\|p-h_{i,\ell}\|_1 \leq 10(\alpha + 2)\opt_{{\cal C}}(p) + \eps$.  The
overall running time and sample complexity are easily seen to be as claimed, and the theorem
is proved.

\end{proof}

\subsection{Dealing with distributions that are not well behaved}
\label{sec:heavyok}

The assumption that the target distribution $p$ is
$\tilde{\Theta}(\eps/k)$-well-behaved can be straightforwardly
removed by following the approach in Section~3.6 of
\cite{CDSS14a}.  That paper presents a simple linear-time sampling-based
procedure, using $\tilde{O}(k/\eps)$ samples, that with high probability
identifies all the ``heavy'' elements (atoms which cause $p$ to not be 
well-behaved, if any such points exist).

Our overall algorithm first runs this procedure to find the set $S$ of
``heavy'' elements, and then runs the algorithm presented above
(which succeeds for well-behaved distributions, i.e., distributions that
have no ``heavy'' elements) using as its target distribution 
the conditional distribution of $p$ over $[0,1) \setminus S$ 
(let us denote this conditional distribution by $p'$).
A straightforward analysis given in \cite{CDSS14a} shows
that (i) $\opt_k(p) \geq \opt_k(p')$, and moreover (ii) 
$\dtv(p,p') \leq \opt_k(p)$.
Thus, by the triangle inequality,
any hypothesis $h$ satisfying $\dtv(h,p') \leq C \opt_k(p') + \eps$
will also satisfy $\dtv(h,p) \leq (C+1) \opt_k(p)+\eps$ as desired.

\ignore{

\red{I think we can just do exactly what we do in the
\cite{CDSS14} paper, and the same analysis given there
yields that it will just cost
us an extra additive $1 \cdot \opt_k(p)$.  Someone else should check to make 
sure this is indeed the case.

If indeed dealing with this annoying issue is 
really just a rehash of how we handle it in 
\cite{CDSS14}, probably we should just say that it can be done
and deal with the whole issue in a few sentences:  say that first we run
a simple procedure that finds heavy elements, then we run the algorithm from
earlier for well-behaved target functions on the conditional distribution 
obtained from the target by conditioning on not drawing an element
identified as heavy.  Probably not even worth stating the exact
performance guarantee of the {\tt Find-Heavy-Elements} procedure.}
}

\section{Lower bounds on agnostic learning} \label{sec:lower}

In this section we establish that $\alpha$-agnostic learning with $\alpha<2$ is information theoretically impossible,
thus establishing Theorem~\ref{thm:main2-informal}.

Fix any $0 < t < 1/2$.  We define a probability distribution  ${\cal D}_{t}$ over a finite set of discrete distributions over the domain $[2N] = \{1,\dots,2N\}$ as follows.  (We assume without
loss of generality below that $t$ is rational and that $tN$ is an integer.)  A draw of
$p_{S_1,S_2,t}$ from ${\cal D}_t$ is obtained as follows.

\begin{enumerate}

\item A set $S_1 \subset [N]$ is chosen uniformly at random from all subsets of $[N]$ that
contain precisely $t N $ elements.  For $i \in [N]$, the distribution $p_{S_1,S_2,t}$
assigns probability weight as follows:
\[
p_{S_1,S_2,t}(i) = 
{\frac 1 {4N}}  \text{~if~}i \in S_1, \quad \quad 
p_{S_1,S_2,t}(i) = 
{\frac 1 {2N}}\left(1 + {\frac t {2(1-t)}}\right)  \text{~if~} i \in [N] \setminus S_1
.
\]

\item A set $S_2 \subset [N+1,\dots,2N]$ is chosen uniformly at random from all subsets of $[N+1,\dots,2N]$ that
contain precisely $ t N $ elements.  For $i \in [N+1,\dots,2N]$, the distribution $p_{S_1,S_2,t}$
assigns probability weight as follows:
\[
p_{S_1,S_2,t}(i) = 
{\frac 3 {4N}} \text{~if~}i \in S_2, \quad \quad 
{\frac 1 {2N}}\left(1 - {\frac t {2(1-t)}}\right)  \text{~if~} i \in [N] \setminus S_1
.
\]

\end{enumerate}

Using a birthday paradox type argument, we show that no $o(\sqrt{N})$-sample
algorithm can successfully distinguish between
a distribution $p_{S_1,S_2,t} \sim {\cal D}_t$ and the uniform distribution
over $[2N]$.
We then leverage this indistinguishability to show that any 
$(2-\delta)$-semi-agnostic learning algorithm, even for $2$-flat distributions,
must use a sample of size $\Omega(\sqrt{N})$:

\begin{theorem} \label{thm:main2-formal}
Fix any $\delta>0$ and any function $f(\cdot)$. There is no algorithm $A$ with the following property:
given $\eps > 0$ and access to independent points drawn from an unknown distribution $p$
over $[2N]$, algorithm $A$ makes $o(\sqrt{N}) \cdot f(\eps)$ draws from $p$ and with probability
at least $51/100$ outputs a hypothesis distribution $h$ over $[2N]$ satisfying
$\|h - p\|_1 \leq (2-\delta)\opt_2(p)+\eps$.
\end{theorem}
\begin{proof}
We write ${\cal U}_{2N}$ to denote the uniform distribution over $[2N]$.  The following proposition
shows that ${\cal U}_{2N}$ has $L_1$ distance from $p_{S_1,S_2,t}$ almost twice that
of the optimal $2$-flat distribution:

\begin{proposition} \label{prop:opt}
Fix any $0 < t < 1/2$.  
\begin{enumerate}

\item 
For any distribution $p_{S_1,S_2,t}$ in the support of ${\cal D}_t$, we have
\[
\| {\cal U}_{2N} - p_{S_1,S_2,t} \|_1 = t.
\]

\item For any distribution $p_{S_1,S_2,t}$ in the support of ${\cal D}_t$, we have
\[
\opt_2(p_{S_1,S_2,t}) \leq {\frac t 2} \left(1 + {\frac t {1-t}}\right).
\]

\end{enumerate}

\end{proposition}

\begin{proof}
Part (1.) is a simple calculation.  For part (2.), consider the 2-flat distribution
\[
q(i) = 
\begin{cases}
{\frac 1 {2N}} \left(1+ {\frac t {2(1-t)}}\right) & \text{~if~}i \in [N]\\
{\frac 1 {2N}} \left(1- {\frac t {2(1-t)}}\right) & \text{~if~}i \in [N+1,\dots,2N]
\end{cases}
\]
It is straightforward to verify that $\|p_{S_1,S_2,t}-q\|_1 =  {\frac t 2} \left(1 + {\frac t {1-t}}\right)$ as claimed.
\end{proof}

 For a distribution $p$
we write $A^p$ to indicate that algorithm $A$ is given access to i.i.d. 
points drawn from $p$.

The following simple proposition states that no algorithm can successfully 
distinguish between
a distribution $p_{S_1,S_2,t} \sim {\cal D}_t$ and ${\cal U}_{2N}$ using
fewer than (essentially) $\sqrt{N}$ draws:

\begin{proposition} \label{prop:indistinguishable}
There is an absolute constant $c>0$ such that the following holds:   Fix any $0 < t < 1/2$, and 
let $B$ be any ``distinguishing algorithm'' which receives $c \sqrt{N}$ i.i.d. draws from a distribution 
over $[2N]$ and outputs either ``uniform'' or ``non-uniform''.  Then
\begin{equation}
\left|
\Pr[B^{{\cal U}_{[2N]}} \text{~outputs ``uniform''}]
-
\Pr_{p_{S_1,S_2,t} \sim {\cal D}_t}[B^{p_{S_1,S_2,t}} \text{~outputs ``uniform''}]
\right|
\leq 0.01. \label{eq:close}
\end{equation}
\end{proposition}
The proof is an easy consequence of the fact that in both
cases (the distribution is ${\cal U}_{[2N]}$, or the distribution is 
$p_{S_1,S_2,t} \sim {\cal D}_{t}$),
with probability at least 0.99 the $c \sqrt{N}$ draws received by $A$ are a uniform
random set of $c\sqrt{N}$ distinct elements from $[2N]$ 
(this can be shown straighforwardly 
using a birthday paradox type argument).

Now we use Proposition \ref{prop:indistinguishable} to show that 
any 
$(2-\delta)$-semi-agnostic learning algorithm even for $2$-flat distributions must use a sample of size $\Omega(\sqrt{N})$,
and thereby prove Theorem~\ref{thm:main2-formal}.

\medskip

Fix a value of $\delta>0$ and suppose, for the sake of contradiction, that there exists such an algorithm $A$.  We describe
how the existence of such an algorithm $A$ yields a distinguishing algorithm $B$ that
violates Proposition \ref{prop:indistinguishable}.

The algorithm $B$ works as follows, given access to i.i.d. draws
from an unknown distribution $p$.
It first runs algorithm $A$ with its ``$\eps$'' parameter set to $\eps := {\frac {\delta^3}{12(2+\delta)}}$, obtaining (with probability
at least $51/100$) a hypothesis distribution $h$ over $[2N]$ such that $\|h - p\|_1 \leq (2-\delta)\opt_2(p)+ \eps.$  It then computes
the value $\|h - {\cal U}_{2N}\|_1$ of the $L_1$-distance between $h$ and the uniform distribution
(note that this step uses no draws from the distribution).  
If $\|h - {\cal U}_{2N}\|_1 < 3\eps/2$ then it outputs ``uniform'' and otherwise
it outputs ``non-uniform.''

Since $\delta$ (and hence $\eps$) is independent of $N$, the algorithm $B$ makes fewer than $c \sqrt{N}$ draws
from $p$ (for $N$ sufficiently large).
To see that the above-described algorithm $B$ violates (\ref{eq:close}), consider first the case that $p$ is 
${\cal U}_{[2N]}$.  In this case $\opt_2(p)=0$ and so with probability at least 51/100 the hypothesis $h$
satisfies $\|h-{\cal U}_{2N}\|_1 \leq \eps$, and hence algorithm $B$ outputs  ``uniform'' with probability
at least $51/100.$

On the other hand, suppose that $p=p_{S_1,S_2,t}$ is drawn from ${\cal D}_t$, where $t={\frac \delta {2+\delta}}$.  
In this case, with probability at least 51/100 the hypothesis $h$ satisfies 
\[
\|h-p_{S_1,S_2,t}\|_1 \leq (2-\delta)\opt_2(p_{S_1,S_2,t}) + \eps \leq (2-\delta)\cdot {\frac t 2} \cdot  \left(1 + {\frac t {1-t}}\right) + \eps,
\]
by part (2.) of Proposition \ref{prop:opt}.  Since by part (1.) of Proposition \ref{prop:opt} we have $\|{\cal U}_{2N} - 
p_{S_1,S_2,t}\|_1 = t$, the triangle inequality gives that
\[
\|h-{\cal U}_{2N} \|_1 \geq t -  (2-\delta)\cdot {\frac t 2} \cdot  \left(1 + {\frac t {1-t}}\right) - \eps = 2\eps,
\] 
where to obtain the final equality we recalled the settings $\eps =  {\frac {\delta^3}{12(2+\delta)}}$,
$t = {\frac \delta {2+\delta}}$.
Hence algorithm $B$ outputs ``uniform''
with probability at most $49/100$.  Thus we have
\[
\left|
\Pr[B^{U_{[2N]}} \text{~outputs ``uniform''}]
-
\Pr_{p_{S_1,S_2,t} \sim {\cal D}_t}[B^{p_{S_1,S_2,t}} \text{~outputs ``uniform''}]
\right| \geq 0.02
\]
which contradicts (\ref{eq:close}) and proves the theorem.
\end{proof}

As described in the Introduction, via the obvious correspondence that maps distributions over $[N]$ to distributions
over $[0,1)$, we get the following:

\begin{corollary} \label{cor:main2-formal}
Fix any $\delta>0$ and any function $f(\cdot)$. There is no algorithm $A$ with the following property:
given $\eps > 0$ and access to independent draws 
from an unknown distribution $p$
over $[0,1)$, algorithm $A$ makes $f(\eps)$ draws from $p$ and with probability
at least $51/100$ outputs a hypothesis distribution $h$ over $[0,1)$ satisfying
$\|h - p\|_1 \leq (2-\delta)\opt_2(p)+\eps$.

\end{corollary}



\bibliographystyle{alpha}
\bibliography{allrefs}


\newpage

\end{document}